\documentclass[11pt]{article}
\usepackage{graphicx} 
\usepackage{preamble}
\usepackage{palatino}

\title{A faster and simpler algorithm for learning shallow networks}
\author{Sitan Chen\thanks{\texttt{sitan@seas.harvard.edu}. Supported by NSF Award \#2103300.} \\ UC Berkeley, Harvard \and Shyam Narayanan\thanks{\texttt{shyamsn@mit.edu}. Supported by an NSF Graduate Fellowship and a Google Fellowship.} \\ MIT}
\date{\today}

\renewcommand{\top}{\intercal}
\newcommand{\relu}{\mathsf{relu}}
\newcommand{\fwlamu}{f_{w,\mathbf{\lambda},\mathbf{u}}}
\renewcommand{\Id}{\mathds{1}}
\newcommand{\He}{\mathrm{He}}
\renewcommand{\epsilon}{\varepsilon}

\begin{document}

\maketitle

\begin{abstract}
    We revisit the well-studied problem of learning a linear combination of $k$ ReLU activations given labeled examples drawn from the standard $d$-dimensional Gaussian measure. Chen et al.~\cite{chen2023learning} recently gave the first algorithm for this problem to run in $\poly(d,1/\epsilon)$ time when $k = O(1)$, where $\epsilon$ is the target error. More precisely, their algorithm runs in time $(d/\epsilon)^{\mathrm{quasipoly}(k)}$ and learns over multiple stages. Here we show that a much simpler one-stage version of their algorithm suffices, and moreover its runtime is only $(d/\epsilon)^{O(k^2)}$.
\end{abstract}

\section{Introduction}

We consider the well-studied problem of PAC learning one-hidden-layer ReLU networks from Gaussian examples. In this problem, there are unknown \emph{weight vectors} $u_1,\ldots,u_k\in\bS^{d-1}$ and \emph{output weights} $\lambda_1,\ldots,\lambda_k\in\R$, and we are given labeled examples $(x_1,f(x_1)),\ldots,(x_N,f(x_N))$ for
\begin{equation}
    f(x) \triangleq \sum^k_{i=1} \lambda_i \,\relu(\iprod{u_i, x})\,,
\end{equation}
where $x_1,\ldots,x_N$ are drawn i.i.d. from the standard $d$-dimensional Gaussian measure $\gamma$. The goal is to output some estimator $\wh{f}$ for which 
\begin{equation}
    \norm{f - \wh{f}}_{L_2(\gamma)} \le \epsilon\,,
\end{equation}
for some target error $\epsilon$. In order for this to be scale-invariant, we adopt the standard (and necessary) normalization convention of assuming that $\sum_i |\lambda_i| \le \mathcal{R}$ for some parameter $\mathcal{R} \ge 1$.

This problem has been a fruitful testbed both for proving rigorous guarantees on training neural networks with gradient descent, and for developing new provably correct algorithms for nonconvex regression in high dimensions. 
While it has been the subject of a long line of work~\cite{janzamin2015beating,sedghi2016provable,bakshi2019learning,ge2018learning,ge2018learning2,sewoong,diakonikolas2020algorithms,zhang2016l1,goel2017reliably,Daniely17, goel2019learning,zhong2017recovery,LiY17,vempala2019gradient,zgu,soltanolkotabi2017learning,zhangps17,diakonikolas2020approximation,LiMZ20,convotron,azll,chen2022learning,smallcovers,chen2023learning}, it remains open to find a $\poly(k,d,\mathcal{R}/\epsilon)$ time algorithm for this problem without making any additional assumptions on the network parameters. For a more thorough overview of related work, we refer the reader to the discussion in~\cite{chen2022learning,chen2023learning}.

Recently, Chen et al.~\cite{chen2023learning} gave the first $\poly(d,\mathcal{R}/\epsilon)$ time algorithm for this problem in the regime where $k = O(1)$. Unfortunately, their dependence on $k$ was rather large, namely $(dk\radius/\epsilon)^{\mathrm{quasipoly}(k)}$. In this work, we obtain the following improvement:

\begin{theorem}\label{thm:main}
    There is an algorithm for PAC learning one-hidden-layer ReLU networks from Gaussian examples with runtime and sample complexity $(dk\mathcal{R}/\epsilon)^{O(k^2)}$.
\end{theorem}

\noindent In~\cite{diakonikolas2020algorithms} (see also~\cite{goel2020superpolynomial}), it was shown, roughly speaking, that any \emph{correlational statistical query} algorithm time at least $d^{\Omega(k)}$ to learn to constant error in this setting. In particular, this lower bound applies to the standard approach in practice of running noisy gradient descent with respect to square loss. The algorithm we use is also a correlational statistical query algorithm, and our Theorem~\ref{thm:main} can thus be interpreted as saying that the lower bound of~\cite{diakonikolas2020algorithms} is qualitatively tight, up to the particular polynomial dependence on $k$ in the exponent. In fact, as we discuss in Remark~\ref{remark:better_degree} at the end of \S\ref{sec:algo_and_analysis}, for the specific hard instance constructed in the lower bound of~\cite{diakonikolas2020algorithms}, the dependence on $k$ in the exponent that our algorithm achieves can actually be improved from quadratic to linear.

\subsection{Comparison to~\cite{chen2023learning}.} 

Our algorithm can be thought of as a simplification of the algorithm proposed by~\cite{chen2023learning} in the following sense. 

The starting point for their algorithm was to form empirical estimates of the moment tensors $T_\ell = \sum_i \lambda_i u_i^{\otimes \ell}$ for various choices of $\ell$ and contract these along a random direction $g\in\bS^{d-1}$ into matrices $M_\ell \triangleq T_\ell(g,\ldots,g,:,:) = \sum_i \lambda_i \iprod{u_i,g}^{\ell - 2} u_iu_i^\top$. Intuitively, these matrices constitute different linear combinations of the projectors $u_iu_i^\top$, and if we take enough different choices of $\ell$, these matrices will collectively span the subspace $\mathrm{span}(u_1u_1^\top,\ldots,u_ku_k^\top)$. So in principle, by taking a suitable linear combination $\sum_\ell \alpha_\ell M_\ell$ of these matrices and computing its top-$k$ singular subspace, we can get access to the subspace spanned by the weight vectors, and then exhaustively enumerate over an epsilon-net over this to find a suitable approximation to the unknown function $f$. 

Unfortunately, as noted in~\cite{chen2023learning}, there are a host of technical hurdles that arise in implementing this strategy, because there might be some weight vectors that are arbitrarily close to each other. \emph{A priori}, this means that for any suitable choice of coefficients $\brc{\alpha_\ell}$, some of the $\alpha_\ell$'s would have to be very (in fact, arbitrarily) large, which would require estimating the moment tensors to arbitrarily small precision. 

Their workaround was to argue that if one takes the top-$k$ singular subspaces of sufficiently many different $M_\ell$'s and computes their joint span $U$, this space is guaranteed to contain a few weight vectors. One can then subtract these from the unknown function and recurse. Unfortunately, the error in estimating weight vectors in each stage of their algorithm compounds exponentially in $k$, and under their analysis, $\Theta(\log k)$ rounds of recursion are needed, which ultimately leads to their $(d/\epsilon)^{\mathrm{quasipoly}(k)}$ runtime.

In the present work, we show that this multi-stage approach is unnecessary, and in fact all of the information needed to reconstruct $f$ is present in the subspace $U$ computed in the first round of their algorithm.\footnote{In fact, we show that it is present even in a certain low-dimensional approximation to this subspace. For technical reasons, it is essential to work with this approximation instead of the full subspace in order to get the claimed $(k/\epsilon)^{k^2}$ dependence in Theorem~\ref{thm:main}, as opposed to a $(k/\epsilon)^{k^4}$ dependence.} The central ingredient in our analysis is a univariate polynomial construction (Lemma~\ref{lem:poly}) that 
shows, roughly speaking, that for any cluster $S\subseteq[k]$ of weight vectors which are $\poly(1/d,1/k,\epsilon/\radius)$-far from all other weight vectors, there exists a linear combination of $M_\ell$'s which is equal to $\sum_{i\in S} \lambda_i u_i u_i^\top$. Crucially, the coefficients in this linear combination can be upper bounded by a quantity depending only on $d,k,\radius/\epsilon$ and not on the distances between the weight vectors. These linear combinations certify that $U$ contains a vector close to each such cluster (Corollary~\ref{cor:ui-close-to-V}), and it is not hard to show (Lemma~\ref{lem:clusters_suffice}) that these vectors are enough to approximate $f$.

\section{Preliminaries}

\paragraph{Notation.} 
Given a positive integer $k$, we use $[k]$ to denote the set of integers $\{1, 2, \dots, k\}$. Likewise, given two positive integers $b \ge a$, we use $[a:b]$ to denote the set of integers $\{a, a+1, \dots, b\}$.

Given functions $a, b: \R_{\ge 0} \to \R_{\ge 0}$, we use $a = O(b)$ and $a \lesssim b$ interchangeably to denote that there exists an absolute constant $C$ such that $a(z) \le C \cdot b(z)$ for all $z$ sufficiently large. 

Given any function $f: \BR^d \to \BR$ and a distribution $\gamma$ over $\BR^d$, we write $\norm{f}_{L_2(\gamma)} = \sqrt{\BE_{x \sim \gamma}[f(x)^2]}$, recalling that $\gamma$ denotes the standard $d$-dimensional normal distribution $\cN(0, \Id)$.

We will always use $u_i$ to denote a vector in the $d$-dimensional unit sphere $\bS^{d-1}$, and $\lambda_i, \mu_i$ to denote real-valued scalars.

For a vector $v$, we use $\|v\|$ to denote its $\ell_2$ norm (or Euclidean norm), and $\|v\|_1$ to denote its $\ell_1$ norm.
Given a real symmetric matrix $M$, we use $\|M\|_{\op}$ to denote its operator norm, and $\|M\|_F$ to denote its Frobenius norm.

\subsection{ReLU networks}

\begin{lemma}[Lemma 2.1 in \cite{chen2023learning}]\label{lem:abs}
    Given  $f = \sum^k_{i=1} \mu_i\,\relu(\iprod{u_i,\cdot})$, there exist $w\in\R^d$ and $\lambda_1,\ldots,\lambda_k\in\R$ such that $f = \iprod{w,\cdot} + \sum^k_{i=1} \lambda_i\,|\iprod{u_i,\cdot}|$.
\end{lemma}

\noindent In light of Lemma~\ref{lem:abs}, given $w\in\R^d$ and $(\lambda_1, u_1),\ldots(\lambda_k,u_k) \in \R\times\bS^{d-1}$, let
\begin{equation}
    \fwlamu(x) \triangleq \iprod{w,x} + \sum^k_{i=1} \lambda_i\,|\iprod{u_i,x}|\,.
\end{equation}

\noindent We will need the following elementary bound relating parameter closeness to closeness in $L_2(\gamma)$ for such functions.

\begin{proposition}[Lemma 3.3, restated, in~\cite{chen2021efficiently}] \label{prop:relu-diff-bound}
    For $x \sim \gamma$ and any unit vectors $u,u'$, 
    \begin{equation}
        \E{(|\langle u, x \rangle| - |\langle u', x \rangle|)^2} \lesssim \|u-u'\|^2\,.
    \end{equation}
\end{proposition}

\subsection{Moment tensors} \label{subsec:moment-tensors}

\noindent Given $g\in\bS^{d-1}$ and $\ell\in\mathbb{N}$, define
\begin{equation}
    T_\ell(\brc{\lambda_i, u_i}) \triangleq \sum^k_{i=1} \lambda_i u_i^{\otimes \ell} \qquad \text{and} \qquad M^g_\ell(\brc{\lambda_i, u_i}) \triangleq \sum^k_{i=1} \lambda_i \iprod{u_i, g}^{\ell - 2} u_i u_i^\top\,, \label{eq:moments}
\end{equation}
noting that the latter can be obtained by contracting the former along the direction $g$ in the first $\ell - 2$ modes, which we denote by $M^g_\ell(\brc{\lambda_i, u_i}) = T_\ell(\brc{\lambda_i, u_i})(g,\cdots,g,:,:)$. When $g$ and $\brc{\lambda_i, u_i}$ are clear from context, we denote these by $T_\ell$ and $M_\ell$ respectively.

These objects can be estimated from samples as follows. Let $\He_\ell(\cdot)$ denote the degree-$\ell$ probabilist's Hermite polynomial. The polynomials $\brc{\frac{1}{\sqrt{\ell!}}\He_\ell}_{\ell \ge 1}$ form an orthonormal basis for the space of functions which are square-integrable with respect to $\gamma$. Define the \emph{normalized Hermite tensor} $S_\ell: \R^d\to(\R^d)^{\otimes\ell}$ to be the tensor-valued function whose $(i_1,\ldots,i_\ell)$-th entry, given input $x\in\R^d$, is $\prod^d_{j=1} \frac{1}{\sqrt{\ell_j!}}\He_{\ell_j}(x_j)$, where $\ell_j$ is the number of occurrences of $j$ within $(i_1,\ldots,i_\ell)$.


\begin{lemma}[Lemma 4.2 in~\cite{chen2023learning}]\label{lem:moment}
    Let $\ell \in \brc{1,2,4,6,\ldots}$ and 
    \begin{equation}
        C_\ell \triangleq \begin{cases}
            1/2 & \text{if} \ \ell = 1 \\
            \frac{\He_\ell(0) + \ell \He_{\ell-2}(0)}{\sqrt{2\pi \ell!}} & \text{if} \ \ell \ \text{even}
        \end{cases}
    \end{equation}
    Let $\eta > 0$. Given samples $\brc{(x_i,\fwlamu(x_i)}_{i\in 1,\ldots,N}$ for $x_i \sim \gamma$ and $N \ge \ell^{O(\ell)}d^{2\ell} \radius^2 / \eta^2$, with high probability the tensor 
    \begin{equation}
        \wh{T} = \frac{1}{2C_\ell N} \sum_i \fwlamu(x_i) \cdot S_\ell(x_i)    
    \end{equation}
    satisfies $\norm{\wh{T} - T_\ell}_F \le \eta$ if $\ell$ is even, and otherwise satisfies $\norm{\wh{T} - w}_2 \le \eta$ if $\ell = 1$. In particular, for even $\ell$, if we define $\wh{M}^g_\ell \triangleq \wh{T}_\ell(g,\cdots,g,:,:)$ then $\norm{\wh{M}^g_\ell - M^g_\ell}_F \le \eta$.
\end{lemma}

\noindent When $g$ and $N$ are clear from context, we will use $\wh{M}_\ell$ to refer to the empirical estimate $\wh{M}^g_\ell$ which is obtained using $N$ samples. We will also use $\wh{w}$ to refer to $\wh{T}$ when $\ell = 1$ to emphasize that it is an empirical estimate of the linear component $w$ in $\fwlamu$.

\subsection{Random contraction}

As in~\cite{chen2023learning}, our algorithm is based on extracting information about the parameters of the network from $\brc{M^g_\ell}$ for a random choice of unit vector $g$. The randomness in $g$ ensures that with high probability, any two weight vectors $u_i, u_j$ are close/far if and only if their projections $\iprod{u_i,g}, \iprod{u_j,g}$ are as well. Formally:

\begin{lemma}[Lemma 2.2 in~\cite{chen2023learning}]\label{lem:anti}
    With probability at least $4/5$ over random $g\in\bS^{d-1}$, for all $i,j$ and $\sigma\in\brc{\pm 1}$,
    \begin{equation}
        \frac{c}{\sqrt{d}}\cdot \frac{1}{k^2} \le \frac{|\iprod{u_i + \sigma u_j, g}|}{\norm{u_i + \sigma u_j}} \le \frac{c'}{\sqrt{d}}\cdot \sqrt{\log k}
    \end{equation}
    for some absolute constants $c,c' > 0$.
\end{lemma}

\noindent Henceforth, we condition on the event that $g$ satisfies Lemma~\ref{lem:anti}. We will denote
\begin{equation}
    z_i \triangleq \iprod{u_i, g}
\end{equation}
and, because of the absolute values in the definition of $\fwlamu$, we may assume without loss of generality that
\begin{equation}
    0 \le z_1 \le \cdots \le z_k\,.
\end{equation}

\subsection{Estimating test error}

Our algorithm will produce a list of many candidate estimates, at least one of which is guaranteed to be sufficiently close in $L_2(\gamma)$ to $\fwlamu$. In order to identify an estimate from the list with this property, we use the following standard result on validation:

\begin{lemma}\label{lem:validate}
    Let $F: \R^d\to\R$ be a $2\radius$-Lipschitz one-hidden-layer ReLU network of size at most $2k$. Let $\delta > 0$, and let $\upsilon > 0$ be a parameter satisfying $\upsilon^2 < 4\mathcal{R}^2 k$. Given $N \gtrsim k^2(\radius / \upsilon)^4 \log(1/\delta)$ samples $x_1,\ldots,x_N\sim\gamma$, we have that
    \begin{equation}
        \Bigl|\E[\gamma]{F^2} - \frac{1}{N}\sum^N_{i=1} F(x_i)^2\Bigr| \le \upsilon^2.
    \end{equation}
\end{lemma}

\noindent We will ultimately take $F$ to be $\fwlamu - \wh{f}$ for various $\wh{f}$ in our list of estimates. All of the $\wh{f}$ we consider will be $\radius$-Lipschitz and have size at most $k$, so $F = \fwlamu - \wh{f}$ will satisfy the hypotheses of Lemma~\ref{lem:validate}.

\section{Polynomial construction}

A key tool in establishing our main result is the following polynomial construction.

\begin{lemma}\label{lem:poly}
    Let $0 < \Delta < 1$, and let $-1 \le x_1 < \cdots < x_k \le 1$. Suppose there are indices $1 \le a < b \le k$ such that $x_{b+1} > x_b + \Delta$ and $x_a > x_{a-1} + \Delta$.
    Then, there exists a degree (at most) $k^2$ polynomial $p$ with coefficients bounded by $O(1/\Delta)^{O(k^2)}$ such that 
    \begin{equation} 
        p(x_s) = \bone{a \le s \le b}
    \end{equation}
    for all $s \in[k]$.
\end{lemma}

\begin{proof}
    Define $I = [a: b]$ to be the set of indices between $a$ and $b$, inclusive. We consider the polynomial
        \begin{equation}
            p(x) = \prod\limits_{j \not\in I} \Bigl(1 - \prod_{i \in I} \frac{x-x_i}{x_j-x_i} \Bigr)\,.\label{eq:pdef}
        \end{equation}
    It is clear that the degree of this polynomial is at most $|I| \cdot (k-|I|) \le k^2$. Next, because every $x_i \in [-1, 1]$ and every $|x_j-x_i| \ge \Delta$, it is clear that $\prod_{j \in I} \frac{x-x_i}{x_j-x_i}$ has all coefficients bounded by $O(1/\Delta)^k$, which means the full polynomial $p(x)$ has all coefficients bounded by $O(1/\Delta)^{k^2}$.

    Next, we evaluate this polynomial on $x_s$, where $s \in I$. In this case, note that $\prod_{i \in I} \frac{x_s-x_i}{x_j-x_i} = 0$ for any $j \not\in I$, because when we set $i = s$, the fraction is $0$. Therefore, $1 - \prod_{i \in I} \frac{x_s-x_i}{x_j-x_i} = 1$ for all $j \not\in I$, so $p(x_s) = \prod_{j \not\in I} 1 = 1$.
    Finally, we evaluate this polynomial on $x_s$, where $s \not\in I$. Note that for $j = s$, $\prod_{i \in I} \frac{x_s-x_i}{x_j-x_i}= 1$, so $1 - \prod_{i \in I} \frac{x_s-x_i}{x_j-x_i} = 0$ for $j = s$. Therefore, $p(x_s) = 0$, because one of the terms in the product that comprises $p$ evaluates to $0$.
\end{proof}

\noindent Lemma~\ref{lem:poly} will end up being applied on $\langle u_i, g \rangle^2$ for some random vector $g$. Using linear combinations of the matrices $\wh{M}_\ell$ described in \S\ref{subsec:moment-tensors}, we can estimate $\sum \lambda_i p(\langle u_i, g \rangle^2) u_i u_i^\top$ for any polynomial $p$. Lemma~\ref{lem:poly} allows us to choose a polynomial $p$ that isolates out a ``cluster'' of somewhat close vectors $\{u_i\}_{i \in I},$ as long as the remaining vectors $u_j$ (for $j \not\in I$) are of distance at least $\Delta$ away. Hence, the linear combination of the matrices $\wh{M}_\ell$ corresponding to this choice of $p$ will result in a matrix which closely approximates the direction of $u_i$ for $i \in I$.

\section{Algorithm and analysis}
\label{sec:algo_and_analysis}

Here we give an analysis for our main algorithm, {\sc NetLearn}, the full specification of which is given in Algorithm~\ref{alg:main}. Roughly speaking, the algorithm proceeds by forming empirical estimates $\wh{M}_\ell$ for the moment matrices $M_{\ell}$ defined in Eq.~\eqref{eq:moments} for $\ell \le O(k^2)$, computing the top singular subspaces of the various $\wh{M}_\ell$, finding an $O(k)$-dimensional approximation $V$ to the collective span of these subspaces, and finally brute-forcing over $V$ to find a sufficiently good estimate for $\fwlamu$.

\begin{algorithm2e}[t]
\DontPrintSemicolon
\caption{\textsc{NetLearn}($f,\epsilon$)}
\label{alg:main}
	\KwIn{Sample access to unknown one-hidden-layer network $\fwlamu$, target error $\epsilon$}
	\KwOut{$\epsilon$-close estimate $\wh{f}$ for $\fwlamu$}
    $\tau \gets \frac{C \eps}{k},$ \hspace{0.5cm} $\xi \gets \frac{C \eps}{k \radius},$  \hspace{0.5cm} $\Delta \gets \frac{C^2 \xi^2 \tau \eps}{2 k^4 d^{3/2} \radius},$  \hspace{0.5cm} $\eta' = \nu \gets \frac{C^2 \xi^2 \tau \Delta^{O(k^2)}}{d \cdot \radius},$  \hspace{0.5cm} $\eta = (\eta')^2$.\;
    $N \gets (k d)^{O(k^2)} \radius^2/\eta^2$, \hspace{0.5cm} $N_{\sf val} \gets O((k \radius/\eps)^4 \log(k \radius/\eps))$.\;
    Form estimates $\brc{\wh{M}_\ell}_{\ell=2,4,\ldots,2k^2+2}$ and $\wh{w}$ from $N$ samples $\{(x_j, y_j)\}_{j=1}^N$. \tcp*{see end of \S\ref{subsec:moment-tensors}} \label{step:moments}
    For each $\ell \in \brc{2,4,\ldots,2k^2+2}$, form the projector $\Pi_\ell$ to the span of the eigenvectors of $\wh{M}_\ell$ with eigenvalue at least $\eta'$ in absolute value. \label{step:Pi-ell}\;
    Compute $\sum_\ell \Pi_\ell$, and let $V$ denote the subspace spanned by the eigenvectors of $\sum_\ell \Pi_\ell$ with eigenvalue at least $\nu$. \label{step:V}\;
    Construct a $\xi/2$-net $N_{\sf u}$ in Euclidean distance over the set of vectors of unit norm in $V$. \label{step:u-net}\;
    Construct a $\xi$-net $N_{\lambda}$ over the interval $[-\radius, \radius]$. \label{step:lambda-net}\;
    Draw $N_{\sf val}$ additional samples $\brc{(x_j, y_j)}_{j=N+1}^{N+N_{\sf val}}$.\;
    \For{$m = 0,1,\ldots,k$}{
        \For{$\wh{\lambda}_1,\ldots,\wh{\lambda}_m\in N_{\lambda}$ and $\wh{u}_1,\ldots,\wh{u}_m\in N_{\sf u}$}{
            \If{$\frac{1}{N} \sum^{N + N_{\sf val}}_{j=N + 1} (y_j - f_{\wh{w},\wh{\lambda},\wh{u}})^2 \le \epsilon^2/2$}{
                \Return{$f_{\wh{w},\wh{\lambda},\wh{u}}$}. \label{step:validate}\;
            }
        }
    }
    \Return{Fail}
\end{algorithm2e}

In the algorithm and analysis, we have several important parameters: $\eta, \eta', \nu, \Delta, \xi,$ and $\tau$. We will not set the exact values of these parameters in the analysis until the end, but we will assume that $\eta \le \eta', \nu \le \Delta \le \xi, \tau \le \eps \le 1$, where we recall that $\eps$ is our desired accuracy.

In \S\ref{sec:cluster}, we introduce some conventions for handling weight vectors which are closely spaced by dividing them up in our analysis into clusters. In \S\ref{sec:pca} we give the main part of our analysis in which we argue that the net $N_{\sf u}$ constructed in {\sc NetLearn} contains vectors close to a subset of weight vectors of the unknown network $\fwlamu$ that could be used to approximate $\fwlamu$ to sufficiently small error. We conclude the proof of Theorem~\ref{thm:main} in \S\ref{sec:final}.

\subsection{Basic clustering}
\label{sec:cluster}

A key challenge in learning one-hidden-layer networks without making any assumptions on the weight vectors is that parameter recovery is impossible, because there may exist weight vectors in the network which are arbitrarily close to each other. In~\cite{chen2023learning}, the authors addressed this by giving a rather delicate clustering-based argument based on grouping together weight vectors that were close at multiple different scales.

In this work, we sidestep this multi-scale analysis and show that under a fixed scale, a naive clustering of the weight vectors suffices for our analysis. Indeed, for a scale $\Delta > 0$ to be tuned later, let $I_1 \sqcup \cdots \sqcup I_m$ be a partition of $[k]$ into disjoint, contiguous intervals such that any adjacent $z_i, z_{i+1}$ in the same interval are at most $\Delta$-apart, whereas the distance between the endpoints of any two intervals exceeds $\Delta$. (Recall that $z_i \triangleq \langle u_i, g \rangle$, where $g$ is a randomly chosen unit vector, and that we assume the indices are sorted in increasing order of $z_i$.) We remark that $I_1,\ldots,I_m$ are only referenced in the analysis, and our actual algorithm does not need to know this partition.

Note that under this partition, any two $z_i, z_{i'}$ in the same interval are at most $k\Delta$-apart. Recalling that we are conditioning on the event of Lemma~\ref{lem:anti}, this implies that for such $i,i'$,
\begin{equation} \label{eq:ui-diff-bound}
    \norm{u_i - u_{i'}} \lesssim \Delta\cdot k^3\sqrt{d}\,.
\end{equation}
In every interval $I_j$, let $i^*_j$ denote its left endpoint. Also define
\begin{equation}
    \overline{\lambda}_j \triangleq \sum_{i \in I_j} \lambda_i\,.
\end{equation}
For a threshold $\tau$ to be tuned later, define
\begin{equation}
    J_{\sf big} \triangleq \brc{j\in[m]: |\overline{\lambda}_j| > \tau}\,.
\end{equation}
Intuitively, $J_{\sf big}$ corresponds to clusters of neurons which are learnable, as the neurons coming from those clusters do not ``cancel'' significantly with each other.

The following shows that a linear combination of projectors to weight vectors from the same cluster is well-approximated by a projector to a single weight vector in that cluster:
\begin{proposition}\label{prop:proj_diff}
    For any $j\in[m]$, and any fixed $i' \in I_j$,
    \begin{equation}
        \Bigl\|\overline{\lambda}_j u_{i'} u_{i'}^\top - \sum_{i\in I_j} \lambda_i u_i u_i^\top \Bigr\|_{\sf op} \lesssim \Delta \cdot k^3\sqrt{d}\, \norm{\lambda}_1\,.
    \end{equation}
\end{proposition}

\begin{proof}
    First, note that by triangle inequality, $\norm{u_{i'} u_{i'}^\top - u_i u_i^\top}_{\sf op} \le \norm{u_{i'} u_{i'}^\top - u_i u_{i'}^\top}_{\sf op} + \norm{u_i u_{i'}^\top - u_i u_i^\top}_{\sf op} = 2 \norm{u_{i'} - u_i},$ since $u_i$ and $u_{i'}$ are both unit vectors.
    Hence, for any $i\in I_j$, we have
    \begin{equation}
        \norm{u_{i'} u_{i'}^\top - u_i u_i^\top}_{\sf op} \le 2 \norm{u_{i'} - u_i} \lesssim \Delta \cdot k^3 \sqrt{d}
    \end{equation}
    by Eq.~\eqref{eq:ui-diff-bound}, from which the claim follows by triangle inequality. 
\end{proof}

\noindent Next, we show that to learn the nonlinear parts of $\fwlamu$, it suffices to estimate $\overline{\lambda}_j$ and $u_{i^*_j}$ for clusters $j\in J_{\sf big}$.

\begin{lemma}\label{lem:clusters_suffice}
    Let $\epsilon > 0$. For sufficiently small constants $c_1, c_2, c_3 > 0$, suppose
    \begin{equation}
        \tau \le \frac{c_1\epsilon}{k}, \ \ \xi \le \frac{c_2\epsilon}{k \cdot \max(1, \norm{\lambda}_1)}, \ \ \ \Delta \le \frac{c_3\epsilon}{k^4\sqrt{d}\,\norm{\lambda}_1}\,. \label{eq:params}
    \end{equation}
    If the parameters $\brc{\wh{\lambda}_j, \wh{u}_j}_{j\in J_{\sf big}} \in \bS^{d-1}\times\R$ satisfy $|\wh{\lambda}_j - \overline{\lambda}_j| \le \xi$ and $\norm{\wh{u}_j - u_{i^*_j}} \le \xi$ for all $j \in J_{\sf big}$, then
    \begin{equation}
        \Bigl\|\sum^k_{i=1} \lambda_i\,|\iprod{u_i,\cdot}| - \sum_{j\in J_{\sf big}} \wh{\lambda}_j\,|\iprod{\wh{u}_j, \cdot}|\Bigr\|_{L_2(\gamma)} \le \epsilon
    \end{equation}
\end{lemma}

\begin{proof}
    For every $j$ (including $j \not\in J_{\sf big}$), we have by triangle inequality and Proposition~\ref{prop:relu-diff-bound},
    \begin{align}
        \Bigl\|\sum_{i\in I_j} \lambda_i\,|\iprod{u_i,\cdot}| - \overline{\lambda}_j\,|\iprod{\wh{u}_j,\cdot}|\Bigr\|_{L_2(\gamma)} &\le \sum_{i\in I_j} |\lambda_i| \cdot \bigl\||\iprod{u_{i},\cdot}| - |\iprod{\wh{u}_j,\cdot}|\bigr\|_{L_2(\gamma)} \\
        &\lesssim \sum_{i \in I_j} |\lambda_i| \cdot (\|u_i-u_{i_j^*}\| + \|u_{i_j^*}-\wh{u}_j\|) \\
        &\lesssim \norm{\lambda}_1 \cdot \left(\Delta \cdot k^3 \sqrt{d} + \norm{u_{i^*_j} - \wh{u}_j}\right) \\
        &\le \norm{\lambda}_1 \cdot (\Delta \cdot k^3 \sqrt{d} + \xi)\,.
    \end{align}
    Furthermore, for all $j \in J_{\sf big}$,
    \begin{equation}
    \bigl\|\overline{\lambda}_j\,|\iprod{\wh{u}_j,\cdot}| - \wh{\lambda}_j\,|\iprod{\wh{u}_j,\cdot}|\bigr\|_{L_2(\gamma)} = |\overline{\lambda}_j-\wh{\lambda}_j| \cdot \big\| |\iprod{\wh{u}_j, \cdot}| \big\|_{L_2(\gamma)} \lesssim \xi\,.
    \end{equation}
    and for all $j \not\in J_{\sf big},$
    \begin{equation}
        \norm{\overline{\lambda}_j\, |\iprod{u_{i^*_j},\cdot}|}_{L_2(\gamma)} \lesssim |\overline{\lambda}_j| \le \tau.
    \end{equation}
By triangle inequality and the fact that $m \le k$, we conclude that
    \begin{equation}
        \Bigl\|\sum^k_{i=1} \lambda_i\,|\iprod{u_i,\cdot}| - \sum_{j\in J_{\sf big}} \wh{\lambda}_j\,|\iprod{\wh{u}_j, \cdot}|\Bigr\| \lesssim \norm{\lambda}_1 \cdot (\xi k + \Delta\cdot k^4\sqrt{d}) + k(\xi + \tau)\,,
    \end{equation}
    and the lemma follows from the bounds in Eq.~\eqref{eq:params}.
\end{proof}

\subsection{Analysis of PCA: Overview}
\label{sec:pca}

Let $\eta$ be a parameter to be tuned later. By Lemma~\ref{lem:moment}, using $N = k^{O(k^2)}d^{O(k^2)}\radius^2/\eta^2$ samples, we can form an empirical estimate $\wh{M}_\ell$ for which $\norm{\wh{M}_\ell - M_\ell}_F \le \eta$, for any positive even $\ell \le O(k^2)$. (We assume WLOG that $\wh{M}_\ell$ is symmetric.) In Line~\ref{step:moments} of {\sc NetLearn}, we do this for all $\ell \in \brc{2,4,\ldots,2k^2+2}$.

For each $\wh{M}_\ell$, we can decompose it as 
\begin{equation}
    \wh{M}_\ell = \sum_{i=1}^d \rho_i^{(\ell)} w_i^{(\ell)} (w_i^{(\ell)})^\top\,,
\end{equation}
where $\rho_i^{(\ell)} \in \BR$ and $w_i^{(\ell)} \in \bS^d$ are the eigenvalues and eigenvectors, respectively, of $\wh{M}_\ell$. In Line~\ref{step:Pi-ell} of {\sc NetLearn}, we compute $\Pi_\ell$ as the projection to the span of the eigenvectors with eigenvalue at least $\eta'$ in absolute value, i.e., 
\begin{equation}
    \Pi_\ell = \sum_{i: |\rho_i^{(\ell)}| \ge \eta'} w_i^{(\ell)} (w_i^{(\ell)})^\top\,. \label{eq:piell_def}
\end{equation} Next, in Line~\ref{step:V} of {\sc NetLearn}, we compute $\sum_\ell \Pi_\ell$, which we can decompose as 
\begin{equation}
    \sum_\ell \Pi_\ell = \sum_{i=1}^d \kappa_i v_i v_i^\top\,, \label{eq:Piell}
\end{equation}
where $\kappa_i, v_i$ are the eigenvalues and eigenvectors, respectively, of $\sum_\ell \Pi_\ell$. We pick $V$ as the span of $v_i$ with $|\kappa_i| \ge \nu.$

In analyzing PCA, we have two main steps. First, in \S\ref{subsec:V-low-dim} we show that $V$ has low dimension. This is because we wish to brute force over choices of $\wh{u}_1, \dots, \wh{u}_m$ in $V$ to find a suitable set of directions. Next, in \S\ref{subsec:ui-in-span-V} we show that every $u_i$, where $i \in j$ for $j \in J_{\sf big}$, is close to $V$. This will allow us to prove that there exists an approximate solution in our brute force search.

\subsection{$V$ has low dimension} \label{subsec:V-low-dim}

Consider any fixed $\ell \in \{2, 4, \dots, 2 k^2 + 2\}$, and consider the empirical estimate $\wh{M}_\ell$ for which $\|\wh{M}_\ell-M_\ell\|_F \le \eta$. 
To bound the dimension of $V$, we first show that every not-too-small eigenvector of $\wh{M}_\ell$ (for all $\ell$) is close to the span of $\{u_i\}_{i=1}^k$.

\begin{lemma} \label{lem:w-close-to-span-U}
    Suppose that $w$ is a (unit) eigenvector of $\wh{M}_\ell$ with eigenvalue at least $\eta'$ in absolute value. Then, $w$ is within Euclidean distance $\eta/\eta'$ of the subspace $\Span(\{u_i\}).$
\end{lemma}

\begin{proof}
    Suppose $\wh{M}_\ell w = \rho w$ for some $|\rho| \ge \eta'$. Note that $\|M_\ell w - \wh{M}_\ell w\| \le \|\wh{M}_\ell-M_\ell\|_{\sf op} \cdot \|w\| \le \eta,$ since $\|w\| = 1$. Hence, $M_\ell w$ is within $\eta$ of $\rho w$. However, note that 
\[M_\ell w = \sum_{i=1}^k \lambda_i \langle u_i, g \rangle^{\ell-2} \cdot u_i u_i^\top w = \sum_{i=1}^k \lambda_i \langle u_i, g \rangle^{\ell-2} \cdot \langle u_i, w \rangle \cdot u_i,\]
    which is in the span of $\{u_i\}_{i=1}^k$. Hence, $\rho w$ is within $\eta$ of $\Span(\{u_i\}),$ and since $|\rho| \ge \eta'$, this means $w$ is within $\eta/\eta'$ of $\Span(\{u_i\}).$
\end{proof}

\noindent Let $U = \Span(\{u_i\})$. Let $\Pi_U$ be the projection matrix onto $U$, and $\Pi_U^\perp$ be the projection matrix onto the orthogonal complement of $U$. Using Lemma~\ref{lem:w-close-to-span-U}, we can bound the inner product between $\Pi_U^\perp$ and the projection matrix $\Pi_\ell$.

\begin{corollary}
    For every $\ell$, $\Tr(\Pi_U^\perp \cdot \Pi_\ell) \le d\, (\eta/\eta')^2$.
\end{corollary}

\begin{proof}
    First, note that $(w_i^{(\ell)})^\top \Pi_U^\perp w_i^{(\ell)} = \|\Pi_U^\perp w_i\|^2,$ which is precisely the squared distance from $w_i^{(\ell)}$ to $\Span(\{u_i\})$. So, if $|\rho_i^{(\ell)}| \ge \eta'$, then $\Tr(\Pi_U^\perp \cdot w_i^{(\ell)} (w_i^{(\ell)})^\top) = w_i^{(\ell)} \Pi_U^\perp w_i \le (\eta/\eta')^2$ by Lemma~\ref{lem:w-close-to-span-U}. Recalling the definition of $\Pi_\ell$ in Eq.~\eqref{eq:Piell}, we obtain the claimed bound.
\end{proof}

\noindent Because there are at most $O(k^2)$ choices of $\ell$, this implies that $\Tr(\Pi_U^\perp \cdot \sum_\ell \Pi_\ell) \le O(d k^2) \cdot (\eta/\eta')^2$. Now, let $\Pi_V$ be the projection matrix to the subspace $V$. We now bound $\Tr(\Pi_U^\perp \cdot \Pi_V)$.

\begin{proposition} \label{prop:projection-product-trace-bound}
    We have that $\Tr(\Pi_U^\perp \cdot \Pi_V) \le O(d k^2 (\eta/\eta')^2/\nu).$
\end{proposition}

\begin{proof}
    Recall that $\sum_{\ell} \Pi_\ell$ has eigendecomposition $\sum_{i=1}^d \kappa_i v_i v_i^\top$.
    Since every $\Pi_\ell$ is positive semidefinite, this means $\kappa_i \ge 0$ for all $i$. Moreover, $\Pi_V = \sum_{i: \kappa_i \ge \nu} v_i v_i^\top,$ which means that $\nu \cdot \Pi_V \preccurlyeq \sum_\ell \Pi_\ell$. Therefore, $\Tr(\Pi_U^\perp \cdot \Pi_V) \le \frac{1}{\nu} \cdot \Tr(\Pi_U^\perp \cdot \sum_\ell \Pi_\ell) \le O(d k^2 (\eta/\eta')^2/\nu)$.
\end{proof}

\noindent Hence, we have the following bound on the dimension of $V$.

\begin{lemma} \label{lem:dim-V-upper-bound}
    We have that $\Tr(\Pi_V) \le k + O(d k^2(\eta/\eta')^2/\nu).$ Hence, the dimension of $V$ is at most $k + O(dk^2 (\eta/\eta')^2/\nu)$.
\end{lemma}

\begin{proof}
    For any projection matrix to a subspace $S$, its trace is the same as the dimension of $S$. So, we just need to prove that $\Tr(\Pi_V) \le k + O(d k^2 (\eta/\eta')^2/\nu).$

    Note that $\Tr(\Pi_V) = \Tr(\Pi_V \cdot (\Pi_U + \Pi_U^\perp)) = \Tr(\Pi_V \cdot \Pi_U) + \Tr(\Pi_V \cdot \Pi_U^\perp).$ Since $\Pi_V, \Pi_U \preccurlyeq I$ as they are projection matrices, $\Tr(\Pi_V \cdot \Pi_U) \le \Tr(\Pi_U) = \dim(U) \le k$. By Proposition~\ref{prop:projection-product-trace-bound}, we have that $\Tr(\Pi_V \cdot \Pi_U^\perp) \le O(d k^2 (\eta/\eta')^2/\nu)$. This completes the proof.
\end{proof}

\subsection{Each important $u_i$ is (almost) in the span of $V$} \label{subsec:ui-in-span-V}

In this subsection, we show that every ``important'' $u_i$ (i.e., where $i \in I_j$ for $j \in J_{\sf big}$) is reasonably close to the span of this subspace $V$.

We recall that $V$ is the subspace found in line 5 of Algorithm~\ref{alg:main}, and that $\Pi_V$ represents the projection matrix to this subspace. We also define $\Pi_V^\perp = I-\Pi_V$ to be the projection matrix to the orthogonal complement of $V$.

First, we show that every $M_\ell$ does not have large inner product with the projection $\Pi_V^\perp$.

\begin{lemma} \label{lem:M-projection-bound}
    For all $\ell \in \{2, 4, \dots, 2k^2+2\}$, we have that $|\Tr(\Pi_V^\perp \cdot M_\ell)| \le d \cdot (\|\lambda\|_1 \cdot \nu + O(\eta')).$
\end{lemma}

\begin{proof}
    Recalling Eq.~\eqref{eq:Piell} and the definition of $V$, we have $\Pi_V^\perp = \sum_{i: \kappa_i < \nu} v_i v_i^\top$. So,
\[\Tr\Bigl(\Pi_V^\perp \cdot \sum_{\ell} \Pi_\ell\Bigr) = \sum_{i: \kappa_i < \nu} \kappa_i \le d \cdot \nu\,.\]
    Next, since $\Pi_V^\perp$ and $\Pi_\ell$ are both positive semidefinite, $\Tr(\Pi_V^\perp \cdot \Pi_\ell) \ge 0$, so for all $\ell$,
\begin{equation} \label{eq:projections-bound}
    \Tr(\Pi_V^\perp \cdot \Pi_\ell) \le d \cdot \nu\,.
\end{equation}

    Note that all the eigenvalues of $M_\ell$ are bounded by $\|\lambda\|_1$ in absolute value since $\langle u_i, g \rangle \le 1$. Because $\norm{\wh{M}_\ell - M_\ell}_{\sf op} \le \norm{\wh{M}_\ell - M_\ell}_F \le \eta$, all of the eigenvalues $\rho_i^{(\ell)}$ of $\wh{M}_\ell$ are bounded by $\|\lambda\|_1 + \eta$ in absolute value. Recalling that $\Pi_\ell$ is the projector to the span of the eigenvectors of $\wh{M}_\ell$ with eigenvalue of magnitude at least $\eta'$, we have that
\begin{equation} \label{eq:projection-M-bound}
    -(\|\lambda\|_1 + \eta) \cdot \Pi_\ell - \eta' \cdot I \preccurlyeq \wh{M}_\ell \preccurlyeq (\|\lambda\|_1 + \eta) \cdot \Pi_\ell + \eta' \cdot I.
\end{equation}

    By combining Equations \eqref{eq:projections-bound} and \eqref{eq:projection-M-bound}, and the fact that $\|M_\ell - \wh{M}_\ell\|_{\sf op} \le \eta$, we have that
\begin{align*}
    |\Tr(\Pi_V^\perp \cdot M_\ell)|
    &\le (\|\lambda\|_1 + \eta) \cdot \Tr(\Pi_V^\perp \cdot \Pi_\ell) + \Tr(\Pi_V^\perp \cdot (\eta+\eta') \cdot I) \\
    &\le d \cdot (\|\lambda\|_1 \cdot \nu + O(\eta'))\,. \qedhere
\end{align*}
\end{proof}

\noindent Next, we show that this implies that for every ``important'' $u_i$, $u_i^\top \Pi_V^\perp u_i$ is small, which will be essential to showing that $u_i$ must be close to the span of $V$. The proof of this will crucially use the polynomial construction from Lemma~\ref{lem:poly}.

\begin{lemma} \label{lem:ui-close-to-V}
    Suppose that $i \in I_j$ for some $j \in J_{\sf big}$. Then, 
    \begin{equation}
        |u_i^\top \Pi_V^\perp u_i| \le \frac{1}{\tau}\bigl[O(1/\Delta)^{O(k^2)} \cdot d \cdot (\|\lambda\|_1 \cdot \nu + \eta') + O(\|\lambda\|_1 \cdot k^3 d^{3/2} \cdot \Delta)\bigr]\,.
    \end{equation}
\end{lemma}

\begin{proof}
    Suppose that $i \in I_j$, and that $p(x) = \sum_{\ell=0}^{k^2} p_\ell x^\ell$ is the polynomial from Lemma~\ref{lem:poly} such that $p(\langle u_i, g \rangle^2) = \bone{i \in I_j}$. Then, 
\[\sum_{\ell = 0}^{k^2} p_\ell M_{2 + 2 \ell} = \sum_{\ell = 0}^{k^2} p_\ell \cdot \sum_{i=1}^k \lambda_i \cdot \langle u_i, g \rangle^{2 \ell} u_i u_i^\top = \sum_{i=1}^k \lambda_i \cdot \sum_{\ell = 0}^{k^2} p_\ell \langle u_i, g \rangle^{2 \ell} u_i u_i^\top = \sum_{i \in I_j} \lambda_i u_i u_i^\top.\]
    Since $|\langle u_i, g \rangle - \langle u_{i'}, g \rangle| \ge \Delta$ for all $i \in I_j, i' \not\in I_j$, and since we are assuming $\langle u_i, g \rangle, \langle u_{i'}, g \rangle \ge 0$, this implies that $|\langle u_i, g \rangle^2 - \langle u_{i'}, g \rangle^2| \ge \Delta^2$ for all $i \in I_j, i' \not\in I_j$.
    Hence, Lemma~\ref{lem:poly} implies every coefficient $p_\ell \le O(1/\Delta^2)^{k^2} = O(1/\Delta)^{O(k^2)}$, so by Lemma~\ref{lem:M-projection-bound} we have
\[\Bigl|\Tr\Bigl(\Pi_V^\perp \cdot \sum_{i \in I_j} \lambda_i u_i u_i^\top \Bigr)\Bigr| \le O(1/\Delta)^{O(k^2)} \cdot d \cdot (\|\lambda\|_1 \cdot \nu + \eta')\,.\]
    By Proposition~\ref{prop:proj_diff}, we have that
\[\left|\Tr(\Pi_V^\perp \cdot \overline{\lambda}_j \cdot u_i u_i^\top)\right| \le O(1/\Delta)^{O(k^2)} \cdot d \cdot (\|\lambda\|_1 \cdot \nu + \eta') + O(\|\lambda\|_1 \cdot k^3 d^{3/2} \cdot \Delta)\,.\]
    Since $j \in J_{\sf big},$ this means $|\overline{\lambda}_j| \ge \tau$, implying the claimed bound on $|u_i^\top \Pi_V^\perp u_i| = |\Tr(\Pi_V^\perp \cdot u_i u_i^\top)|$.
\end{proof}

\noindent As a corollary, we have that $u_i$ is close to the span of $V$.

\begin{corollary} \label{cor:ui-close-to-V}
    For any $i \in I_j$ where $j \in J_{\sf big}$, the distance from $u_i$ to $V$ is at most
    \begin{equation}
        \tau^{-1/2}\sqrt{O(1/\Delta)^{O(k^2)} \cdot d \cdot (\|\lambda\|_1 \cdot \nu + \eta') + O(\|\lambda\|_1 \cdot k^3 d^{3/2} \cdot \Delta)}\,.
    \end{equation}
\end{corollary}

\begin{proof}
    We write $u_i = \Pi_V u_i + \Pi_V^\perp u_i$. Note that $\Pi_V u_i \in \Span(V)$, so we just need to bound $\|\Pi_V^\perp u_i\|$. But since $\Pi_V^\perp$ is a projection matrix, $\|\Pi_V^\perp u_i\|^2 = u_i^\top \Pi_V^\perp u_i.$ The claim then follows by the bound on $u_i^\top \Pi^\perp_V u_i$ in Lemma~\ref{lem:ui-close-to-V}.
\end{proof}

\subsection{Putting everything together}
\label{sec:final}

We recall that $\mathcal{R} \ge 1$ is a promised upper bound for $\|\lambda\|_1$. 
For small constant $C > 0$, take
\begin{equation}
    \tau = \frac{C \epsilon}{k}, \qquad
    \xi = \frac{C \epsilon}{k \cdot \mathcal{R}}, \qquad
    \Delta = \frac{C^2 \xi^2 \tau \cdot \epsilon}{2k^4d^{3/2} \cdot \mathcal{R}}, \qquad
    \eta' = \nu = \frac{C^2 \xi^2 \tau \cdot \Delta^{O(k^2)}}{d \cdot \mathcal{R}}, \qquad
    \eta = (\eta')^2\,.
\end{equation}

Under these parameter settings, by Lemma~\ref{lem:dim-V-upper-bound} the dimension of $V$ is at most $k + O(d \cdot k^2 \cdot \nu)$, since $\eta' = \nu$ and $\eta = (\eta')^2$. However, $\nu \le C^2 \cdot \xi^2/d \le C^2/(k^2 d)$. Therefore, if $C$ is sufficiently small, the dimension of $V$ is at most $k + 0.1$, so is at most $k$.
Next, by Corollary~\ref{cor:ui-close-to-V} every $u_i$ for $i \in I_j, j \in J_{\sf big}$ has distance at most
\[\sqrt{\frac{O(C^2 \xi^2 \tau)}{\tau}} \le \frac{\xi}{4}\]
to $V$. For $C$ sufficiently small, we also have that $\tau, \xi, \Delta$ satisfy the constraints of Lemma~\ref{lem:clusters_suffice}. Finally, it is straightfoward to verify that $1/\eta = (d k \radius/\eps)^{O(k^2)}$.

Now, for each $j \in J_{\sf big}$, we recall the definition of $\bar{\lambda}_j$. Since $|\bar{\lambda}_j| \le \mathcal{R}$ by our assumption that $\sum |\lambda_i| \le \mathcal{R}$, by the definition of $N_\lambda$ (see Line~\ref{step:lambda-net} of Algorithm~\ref{alg:main}), there exists $\wh{\lambda}_j \in N_\lambda$ within distance $\xi$ of $\bar{\lambda}_j$. Next, $u_{i_j^*}$ has distance at most $\frac{\xi}{2}$ to $V$, and importantly, $\|u_{i_j^*} - \Pi_V u_{i_j^*}\| \le \frac{\xi}{4}$ and so $|\|\Pi_V u_{i_j^*}\| - 1| \le \frac{\xi}{4}$. So, for $u = \frac{\Pi_V u_{i_j^*}}{\|\Pi_V u_{i_j^*}\|}$, $\|u_{i_j^*} - u\| \le \frac{\xi}{2}$. Therefore, by the definition of $N_{\sf u}$ (see Line~\ref{step:u-net} of Algorithm~\ref{alg:main}), there exists $\wh{u}_j$ within distance $\xi$ of $u_{i_j^*}$.

Therefore, our algorithm will find some $m_{\sf big} = |J_{\sf big}|$ and $(\wh{\lambda}_1, \wh{u}_1), \dots, (\wh{\lambda}_{m_{\sf big}}, \wh{u}_{m_{\sf big}})$ satisfying the conditions of Lemma~\ref{lem:clusters_suffice}, which will thus be within $\eps$ of the true answer in the distance $\|\cdot\|_{L_2(\gamma)}$.

The runtime is dominated by the time it takes to estimate $\{M_\ell\}$ and $w$, which requires 
\begin{equation}
    (k^2 d)^{O(k^2)} \cdot \radius^2/\eta^2 = (dk \radius/\epsilon)^{O(k^2)}
\end{equation}
samples by Lemma~\ref{lem:moment}, and the time it takes to enumerate over sets of at most $k$ vectors from $N_{\sf u}$, and over weights from $N_{\lambda}$, which are of size
\begin{equation}
    |N_{\sf u}| \le O(1/\xi)^{O(k)} = O(k \radius / \epsilon)^{O(k)} \qquad \text{and} \qquad |N_{\lambda}| \le O(k\radius/ \epsilon)^{O(k)}\,.
\end{equation}
We remark that the net $N_{\sf u}$ can be algorithmically constructed by selecting random points. Indeed, for any point $x$ on a $k$-dimensional sphere, a random point on the sphere is within $\xi/2$ of $x$ with probability at least $\Omega(\xi)^{k},$ so for any $\xi/2$-net $N_{\sf u}^*$ of size $O(1/\xi)^{k},$ if $N_{\sf u}$ is constructed as $O(1/\xi)^{2k}$ random points on the sphere, then every point in $N_{\sf u}^*$ will be within $\xi/2$ of at least one point in $N_u$ with high probability. So, $N_u$ is a $\xi$-net of the $k$-dimensional unit sphere.
Hence, we enumerate over at most $(|N_u| \cdot |N_\lambda|)^{O(k)} \le O(k \cdot \radius/\eps)^{O(k^2)}$ candidate solutions. 

By Lemma~\ref{lem:validate}, our algorithm will successfully verify any candidate solution $\{(\wh{\lambda}_i, \wh{u}_i)\}_{i=1}^m$ for each $0 \le m \le k$, using $O(k^2 (\radius/\eps)^4 \log(1/\delta))$ samples. If we set $\delta = \Theta(1/(|N_\lambda| \cdot |N_u|))^k,$ then by a union bound the algorithm will successfully verify every candidate solution, and thus will succeed. Hence, we need to only draw $O(k^4 (\radius/\eps)^4 \cdot \log (k \radius/\eps))$ samples.

This yields the claimed time/sample complexity bound of $(d k \radius/\epsilon)^{O(k^2)}$.

\begin{remark} \label{remark:better_degree}
    As the above proof makes clear, the quadratic dependence on $k$ in the $d^{O(k^2)}$ runtime for {\sc NetLearn} comes from the degree of the polynomial construction in Lemma~\ref{lem:poly}, and the final brute force search (the latter only contributes to runtime, not to sample complexity). Note, however, that the upper bound of $k^2$ on the degree in this construction is somewhat pessimistic. Recall from the proof of Lemma~\ref{lem:poly} that the polynomial $p$ defined in Eq.~\eqref{eq:pdef} has degree at most $|I| \cdot (k - |I|)$, where $I$ is any one of the clusters of neurons indexed by $j\in J_{\sf big}$. In particular, if each such $I$ is of constant size, e.g., then the degree of $p$ is actually $O(k)$, and the dimension dependence in the sample complexity of {\sc NetLearn} (and the runtime barring the final brute force search) improves to $d^{O(k)}$.

    One simple situation in which this happens is if all of the weight vectors are $\Delta' = \poly(\epsilon/\radius,1/k,1/d)$-separated, in which case we can tune the $\Delta$ parameter in the analysis appropriately to ensure that all of the relevant intervals $I$ are of size $1$. For example, in the hard instance in the correlational statistical query lower bound of~\cite{diakonikolas2020algorithms}, the weight vectors are of the form $u_j = \cos(\pi j/k) \cdot v + \sin(\pi j/k) \cdot w$ for two orthogonal unit vectors $v,w$ (see Eq. (3) therein) and are thus $\Omega(1/k)$-separated. Moreover, the span of the $\{u_j\}$ vectors has rank $2$, and the proof of Lemma~\ref{lem:dim-V-upper-bound} implies the dimension of $V$ is at most $2$ as well. Thus, the runtime of brute force can also be reduced to exponential in $k$, rather than in $k^2$. Our algorithm is a correlational statistical query algorithm,\footnote{Note that we need to estimate $\E{(y - \wh{f}(x))^2}$ for various choices of $\wh{f}$ in Line~\ref{step:validate} of {\sc NetLearn}, and technically this is not a correlational statistical query. We can nevertheless remedy this by instead estimating $\E{2y \cdot \wh{f}(x) - \wh{f}(x)^2}$ and outputting the estimator $\wh{f}$ which maximizes this quantity.} and on this instance it has runtime which matches the lower bound.
\end{remark}

\newcommand{\etalchar}[1]{$^{#1}$}

\end{document}